\documentclass[letterpaper, 10pt, conference]{ieeeconf}  %

\IEEEoverridecommandlockouts %
\usepackage[utf8]{inputenc}
\usepackage[T1]{fontenc}

\usepackage{amsmath,amssymb,amsthm}
\usepackage{amsfonts}
\usepackage{bm}
\usepackage{graphicx}
\usepackage{cuted} %
\usepackage{subcaption}
\usepackage{booktabs}
\usepackage{nicefrac}
\usepackage{microtype}
\usepackage[table]{xcolor}
\usepackage{url}
\usepackage{xspace}

\makeatletter
\let\NAT@parse\undefined %
\makeatother
\usepackage{cite}

\usepackage{hyperref}

\hypersetup{
    colorlinks=true,  %
    linkcolor=black,  %
    citecolor=black,  %
    urlcolor=black    %
} %
\usepackage{cleveref}

\usepackage{algorithm}
\usepackage{algpseudocode}

\newtheorem{proposition}{Proposition}

\newtheorem{corollary}{Corollary}

\title{\LARGE \bf Training-free Behavior Cloning}%

\author{Maximilian Adang$^{*}$, Timothy Chen$^{*}$, Lars Osterberg, Aiden Swann, and Mac Schwager
\thanks{$^{\dagger}$ Stanford University, Stanford, CA 94404, USA. {\tt\small \{madang, chengine, larso33, swann, schwager\}@stanford.edu}}
\thanks{*Equal contribution}
\thanks{$\ddagger$ This work was supported in part by ONR project N00014-23-1-2354. Toyota Research Institute provided funds to support this work. Maximilian Adang is supported by the Air Force Office of Scientific Research under award number FA9550-23-F-0014. Timothy Chen is supported by the NASA NSTGRO fellowship. Aiden Swann is supported by NSF Graduate Research Fellowship DGE-2146755. We are grateful for this support.}
}

\begin{document}
\bstctlcite{IEEEexample:BSTcontrol} %
\bstctlcite{IEEEctldash}
\maketitle

\begin{strip}
  \par\vspace{-6\baselineskip}
  \centering
    \captionsetup{type=figure,font=footnotesize}
  \includegraphics[width=1.0\textwidth]{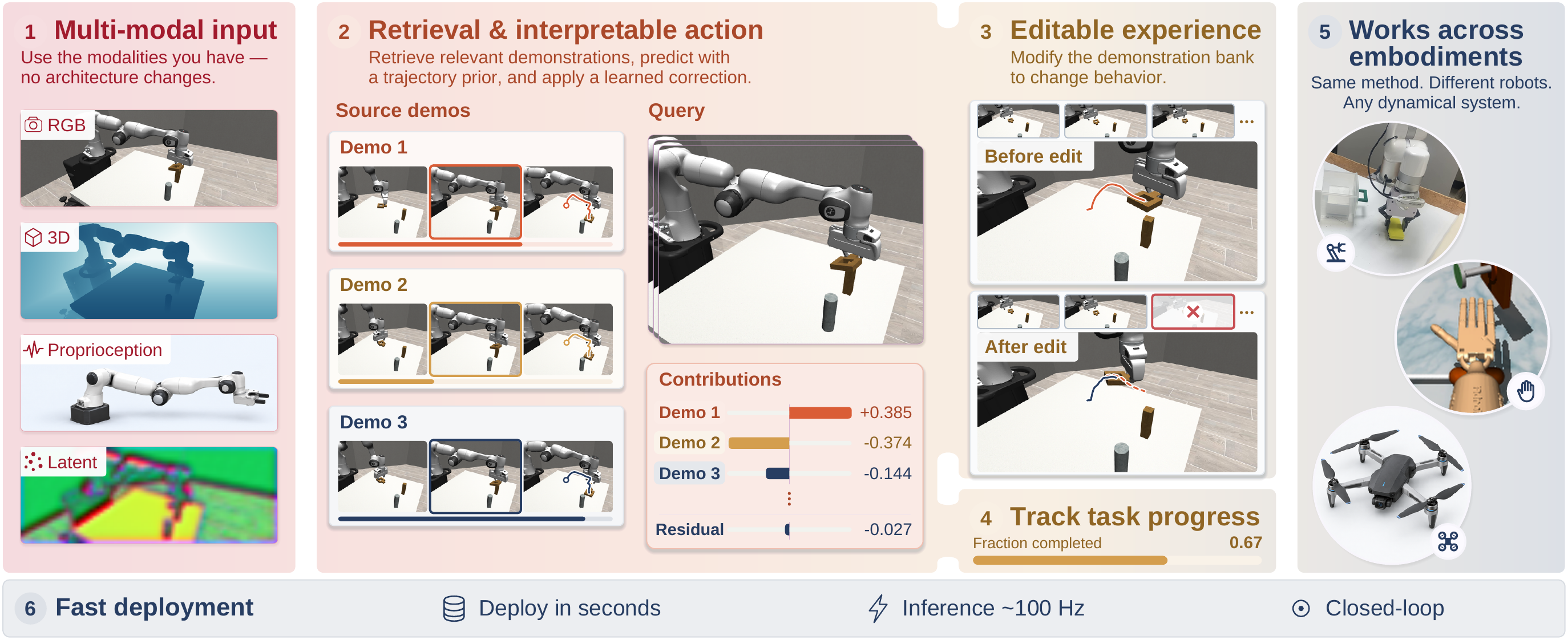}\par %
  \caption{\textbf{Behavior Predictive Control} proposes an alternative paradigm to behavior cloning by extending the capabilities of retrieval policies. Unlike other behavior cloning policies such as Diffusion Policy or VLAs, Behavior Predictive Control (BPC) stores raw imitation data, and blends this data at runtime for action prediction. BPC reduces the time between data collection and deployment from hours to seconds and facilitates execution on resource-constrained hardware. BPC implements a direct algorithmic link between data and inference --- allowing us to causally trace policy outputs back to the source training data, edit that data, and track ongoing task progress without any form of training.}
  \label{fig:bpc_overview}
  \vspace{-10pt}
\end{strip}

\begin{abstract}
Neural behavior cloning compresses demonstrations into large models, making individual actions difficult to trace and policy updates costly. Retrieval policies retain access to demonstrations but struggle with mismatch between recorded and live behavior. We introduce Behavior Predictive Control (BPC), which synthesizes policies without end-to-end policy training by combining an action-aware retrieval metric, a Hankel-based action-continuation prior, and a closed-form one-step residual correction. Inspired by behavioral systems theory, BPC predicts future actions by blending stored observation-action data that best reconstructs the recent runtime observation--action history. Across simulated benchmarks and real-robot deployments, BPC is competitive with learned policies such as $\pi_{0.5}$ (surpassing it in some cases), while reducing policy fitting from hours to seconds on consumer GPUs and supporting closed-loop control upwards of 75 Hz on a Jetson Orin Nano. The retrieved demonstration windows and their coefficients also provide an intrinsic estimate of task progress. Retaining demonstrations within the deployed policy makes its predictions traceable to supporting trajectories and enables behavior revision through the demonstration bank.
\end{abstract}

\section{Introduction}

Large datasets and pretrained representations have driven advances in robot
imitation learning. Diffusion~\cite{chi2023diffusion},
flow-matching~\cite{black2024pi0,zhang2024flowpolicy}, and
vision-language-action (VLA) policies~\cite{brohan2023rt2,kim2024openvla}
fit neural action predictors to expert demonstrations, often using pretrained
visual or vision-language backbones. Their success on manipulation and
other robot tasks has made neural behavior cloning a powerful approach to
learning skills from examples.

This approach nevertheless complicates rapid deployment and behavior revision, and restricts access to robot intelligence.
Fitting or adapting a policy can require hours of GPU computation, making
frequent updates costly and often requiring server-based compute. Once demonstrations are compressed into network
weights, individual actions are difficult to attribute to their supporting
examples, and deleting an undesirable demonstration does not remove its
influence without updating the fitted model. Generalization to new settings
also remains task-dependent. Lookup-based diagnostics show that direct
trajectory retrieval can be competitive with diffusion policies~\cite{he2025demystifying}, motivating a complementary
approach that retains demonstrations at deployment. Our goal is to make
robot capabilities easier to deploy and revise when compute and fitting
time are limited.

Retrieval policies keep demonstrations explicit, select relevant behavior
at inference time, and adapt it to the live
query~\cite{pari2022vinn,he2025demystifying,pfeifer2026darp}. This exposes
the supporting data and makes the demonstration bank a direct point of
intervention. The central challenge is query--demonstration mismatch: the
robot's evolving observation--action history need not coincide exactly with a
recorded trajectory. Selecting or averaging nearby actions does not
explicitly reconcile that history with the behavior that follows it. We
instead use the retrieved trajectories as a local model for action
continuation, reconstructing the observed past and carrying the resulting
coefficients into the demonstrated future.

Behavioral systems theory~\cite{willems2005,markovsky2021tutorial} provides
the motivation for this construction. Under appropriate linearity and
excitation assumptions, a Hankel matrix of recorded data spans a system's
valid trajectories. Data-enabled Predictive Control
(DeePC)~\cite{coulson2019deepc,coulson2019regularized} uses this
representation to solve a reference-tracking problem to optimize future behavior consistent with an observed
history, without having to identify a parametric dynamics model. Behavior
cloning poses a different question: which expert action should follow the
current history, without a supplied reference trajectory or reward? Viewing
demonstrations as recordings of an expert-controlled system turns this
question into trajectory completion. Expert data need not satisfy the
fundamental lemma's assumptions, so the linear representation motivates,
rather than guarantees, continuation from arbitrary demonstrations.

We introduce Behavior Predictive Control (BPC), which combines a fitted
retrieval metric, a regularized Hankel-based action-continuation prior, and
a closed-form one-step residual correction. The metric selects windows
whose histories are informative about future actions. BPC then reconstructs
the live history over those windows and uses the same coefficients to
predict their action continuation; a residual fitted on the same support
corrects the action to account for nonlinear or higher-order effects. Frozen pretrained visual features extend the
representation to visuomotor control. In this work, we find that locality and regularization can address the 
nonlinear regime, and provide analysis relating next-action error to history
reconstruction and local approximation errors.

BPC constructs policies without end-to-end policy training while retaining
their supporting trajectories. The retrieved window positions and
coefficients also provide an intrinsic estimate of task progress (\Cref{fig:bpc_overview}). Simulation
and hardware results show competitive performance, generalizing across tasks and embodiments.
Reported configurations reduce policy fitting from hours to seconds on
consumer GPUs (and works even on onboard robot hardware), and support control
at 75 Hz on a Jetson Orin Nano. This combination makes demonstration-based
policy construction practical with modest compute while preserving access
to the data that determines behavior.

\section{Related Work}
\begin{table*}[t]
\centering
\caption{Policy archetype tradeoffs.}
\label{tab:policy-tradeoffs}
\footnotesize
\setlength{\tabcolsep}{3pt}
\renewcommand{\arraystretch}{1.05}
\setlength{\extrarowheight}{0pt}
\definecolor{policyaccent}{HTML}{243F68}
\begin{tabular*}{\textwidth}{
  >{\raggedright\arraybackslash}p{\dimexpr0.14\textwidth-2\tabcolsep\relax}
  >{\centering\arraybackslash}p{\dimexpr0.105\textwidth-2\tabcolsep\relax}
  >{\raggedright\arraybackslash}p{\dimexpr0.175\textwidth-2\tabcolsep\relax}
  >{\raggedright\arraybackslash}p{\dimexpr0.135\textwidth-2\tabcolsep\relax}
  >{\raggedright\arraybackslash}p{\dimexpr0.19\textwidth-2\tabcolsep\relax}
  >{\centering\arraybackslash}p{\dimexpr0.12\textwidth-2\tabcolsep\relax}
  >{\centering\arraybackslash}p{\dimexpr0.135\textwidth-2\tabcolsep\relax}}
\toprule
\rowcolor{policyaccent!9}
\textbf{Policy family} &  \textbf{Fit time} &
\textbf{Inference} &  \textbf{Memory} &
\textbf{Pretraining?} &  \textbf{Performance} &  \textbf{Interpretable?} \tabularnewline
\midrule
 \textbf{Learned}\newline\cite{chi2023diffusion,zhao2023act,kim2024openvla}
&  Hours
&  Slow; real-time may need action chunking
&  High; may need server GPUs
&  Frequently
&  Strong
&  Opaque\tabularnewline
\addlinespace[2pt]
 \textbf{Retrieval}\newline\cite{pari2022vinn,pfeifer2026darp}
&  Minutes
&  Fast
&  Low
&  Pretrained encoder required
&  Lower
&  High\tabularnewline
\addlinespace[2pt]
\rowcolor{policyaccent!5}
\textbf{BPC (ours)}
&  Seconds
&  Fast; $75$ Hz on Jetson Orin Nano
&  Low
&  Optional; raw pixels may suffice
&  Strong
&  High \tabularnewline
\bottomrule
\end{tabular*}
\end{table*}

\paragraph*{Behavioral systems theory and trajectory completion}
Willems' fundamental lemma~\cite{willems2005} motivates representing
trajectories through Hankel data matrices under linearity and excitation
assumptions. DeePC uses this representation for predictive
control~\cite{coulson2019deepc}, with regularization addressing uncertainty
and robustness~\cite{coulson2019regularized,huang2021quadratic}.
Trajectory-template methods~\cite{depersis2026templates} frame prediction
as completing a trajectory from its observed past. BPC adapts this viewpoint
to expert action continuation: retrieved demonstration windows pair an
observed history with the actions that follow it. In fact, BPC is derived constructively, grounded in nonlinear analysis
relating next-action error to history reconstruction and local approximation, rather than assuming automatic satisfaction of the fundamental lemma. 

\paragraph*{Local and latent DeePC} Previous works have shown that locality and compression allow methods based on the fundamental lemma to extend to nonlinear scenarios.
SVD compression~\cite{zhang2022dimreduction} and sensitivity or datamodel-based column selection~\cite{zhang2025rds,taskaware2025} reduce the cost of data-driven control. Nonlinear extensions use basis lifts~\cite{lazar2023basis}, kernels~\cite{huang2022kernelized}, or learned operators~\cite{deepdeepc2024}. Neural DeePC~\cite{lazar2024neuraldeepc} learns a neural feature basis from system trajectories for online affine interpolation. BPC is flexible in its inputs (from pretrained encoders to raw pixels) and selects demonstration windows through a history distance fitted using demonstrated continuations.

\paragraph*{Parametric behavior cloning} 
MLP action regressors~\cite{pmlr-v164-mandlekar22a}, Diffusion Policy~\cite{chi2023diffusion}, ACT~\cite{zhao2023act}, flow-based policies~\cite{black2024pi0,zhang2024flowpolicy}, and VLAs~\cite{brohan2023rt2,kim2024openvla} fit neural action predictors from demonstrations and have been applied to tabletop manipulation, dexterous manipulation~\cite{wang2024sdp}, and drone navigation~\cite{low2025sous,tucker2026pi}. Lookup-based diagnostics~\cite{he2025demystifying} show
that direct trajectory retrieval can be competitive with diffusion policies, motivating explicit use of demonstrations at deployment. BPC retains this access while fitting a lightweight retrieval metric and residual, with its action prior computed directly from local trajectory data. 

\paragraph*{Retrieval-based imitation and residual correction}
VINN~\cite{pari2022vinn} retrieves actions through visual embeddings, while Behavior Retrieval~\cite{du2023retrieval} selects relevant data to support few-shot imitation. DARP~\cite{pfeifer2026darp} conditions a learned policy on retrieved expert states, actions, and differences from the query. 
BPC combines local retrieval with explicit history reconstruction, and we show empirically and theoretically that this nonparametric history conditioned policy is more robust than common alternative nonparametric priors such as uniform and weighted $k$-nearest-neighbors. Table~\ref{tab:policy-tradeoffs} summarizes the practical tradeoffs between
learned policies, retrieval policies, and BPC.

\makeatletter
\@ifundefined{bpcproposition}{\newtheorem{bpcproposition}{Proposition}}{}
\@ifundefined{bpccorollary}{\newtheorem{bpccorollary}[bpcproposition]{Corollary}}{}
\@ifundefined{bpcproblem}{
  \theoremstyle{definition}
  \newtheorem{bpcproblem}{Problem}
  \theoremstyle{plain}
}{}
\makeatother

\section{Regularized Local Action Continuation}
\label{sec:bpc-preliminaries}

We adapt the trajectory-completion problem of De Persis and
Tesi~\cite[Sec.~II-A, Problem~1]{depersis2026templates} to behavior cloning:
demonstration data constitutes a library of behaviors. Rather than predicting \emph{future observations} based on a history of observations and executed actions, we seek to predict the expert \emph{action} that should follow the observed
history. 

\paragraph*{Generating behavior}
Time indices use parenthesized superscripts; library-column and
demonstration indices use subscripts. Consider a closed-loop plant operated by an unknown
expert
\begin{equation}
\begin{aligned}
x^{(t+1)}&=f_{\mathrm{CL}}(x^{(t)}) = f_{\mathrm{plant}}(x^{(t)}, u^{(t)}) ,\\
 y^{(t)}&=f_y(x^{(t)}),\qquad u^{(t)}=f_u(x^{(t)}),
\end{aligned}
\label{eq:bpc-generating-behavior}
\end{equation}
where $y^{(t)}\in\mathbb R^{n_y}$ is the observation representation
and $u^{(t)}\in\mathbb R^{n_u}$ is the expert command. None of these mappings or
hidden states $x^{(t)}$ is accessible.

\paragraph*{Demonstration templates}
The available data are finite observation--action recordings
$\mathcal D=\{(y_d^{(t)},u_d^{(t)})_{t=0}^{T_d-1}\}_{d=1}^{N_{\mathrm{demo}}}$.
For a history length $T_{\mathrm{ini}}$ and continuation horizon $T_f$, split
an admissible window at decision time $t$ into
\begin{equation}
\begin{aligned}
 z^{(t)}&=\operatorname{col}\bigl(
 u^{(t-T_{\mathrm{ini}}:t-1)},\,
 y^{(t-T_{\mathrm{ini}}+1:t)}\bigr),\\
 u_f^{(t),\star}&=\operatorname{col}(u^{(t),\star},\ldots,u^{(t+T_f-1),\star}),\\
 w^{(t),\star}&=\operatorname{col}(z^{(t)},u_f^{(t),\star}),\qquad
 m_p=T_{\mathrm{ini}}(n_u+n_y).
\end{aligned}
\label{eq:bpc-problem-window}
\end{equation}
Here $\operatorname{col}$ stacks vectors in the displayed order; the star
marks the expert continuation to be inferred. Concatenating the
per-demonstration Hankel blocks, without crossing demonstration boundaries,
gives the library
\begin{equation}
 \mathcal H_{\mathcal D}=[w_1\ \cdots\ w_N]
 =\begin{bmatrix}H\\ U\end{bmatrix}
 \in\mathbb R^{(m_p+T_f n_u)\times N}.
\label{eq:bpc-problem-library}
\end{equation}

\begin{bpcproblem}[Behavior cloning from trajectory templates]
\label{prob:bpc-action-completion}
Given demonstrations $\mathcal D$, organized into the library
$\mathcal H_{\mathcal D}$, and an observed history
$z^{(t)}$ (which we call the \emph{query}), construct a policy
\begin{equation}
 \widehat\pi_{\mathcal D}:\mathbb R^{m_p}\to\mathcal U,\qquad
 \widehat u^{(t)}=\widehat\pi_{\mathcal D}(z^{(t)}),
\label{eq:bpc-problem-policy}
\end{equation}
where $\mathcal U$ is the admissible action set, to estimate
$u^{(t),\star}=u_f^{(t),\star}\left[1{:}n_u\right]$.
For the stacked future bank, $U^{(k)}$ denotes the $k$th future-time block
(the $n_u$ rows for future step $k$, across all columns). Thus the
one-step action bank is $A=U^{(1)}$. Bracket indices select vector entries.
The goal is to reduce the next-action error
$\|\widehat u^{(t)}-u^{(t),\star}\|$. At deployment, the demonstrations and measured
history are available, but the query's future, an expert-action oracle, a
reward, and a reference trajectory are not. We assume that a compatible
history identifies the expert's next action. This identifiability assumption
does not imply that the library covers every history visited by the learner.
\end{bpcproblem}

\subsection{Policy construction}
Let $z\in\mathbb R^{m_p}$ be the query and let $h_i$ and
$u_i^+\in\mathbb R^{T_f n_u}$ be the history and continuation of each of $K$
selected templates. Locally, $H$, $U$, and $A$ denote these submatrices:
\begin{equation}
\begin{aligned}
 H&=[h_1\ \cdots\ h_K],& U&=[u_1^+\ \cdots\ u_K^+],\\
 A&=U^{(1)},&g_0&=\mathbf1/K.
\end{aligned}
\label{eq:bpc-local-dictionary}
\end{equation}
Selecting one demonstration, averaging several, and fitting a local continuation all amount to choosing coefficients \(g\). These reconstruct the history \(Hg\) and predict the continuation \(Ug\), whose first action is \(Ag\). We express these choices as

\begin{equation}
\label{eq:base-policy}
\begin{aligned}
 g^\star&\in\underset{\mathbf 1^\top g=1}{\operatorname{argmin}}
 \;\frac12\|H g-z\|^2+\lambda_g\|g\|,\\
 G(z)&=A g^\star,\qquad \lambda_g\geq0,
\end{aligned}
\end{equation}
which is similar in form to constraints imposed by \cite{depersis2026templates, coulson2019deepc}. The following discussion illuminates why certain constraints on the continuation (in particular affine ones) may perform better than their nearest neighbor counterparts.

\subsection{Local approximation with signed coefficients}
\label{subsec:bpc-local-bound}
\begin{proposition}[Local continuation error]
\label{prop:bpc-local-error}
Let $\pi$ be continuously differentiable on an open convex neighborhood
containing $z$ and the retrieved histories, with $\beta$-Lipschitz Jacobian. For any $g\in\mathbb R^K$, define
\begin{equation}
 e=Hg-z,\quad s=\mathbf1^\top g,\quad
 J_z=D\pi(z),\quad b_z=\pi(z)-J_z z.
\label{eq:bpc-error-quantities}
\end{equation}
Then
\begin{equation}
\begin{aligned}
 \|Ag-\pi(z)\|\leq{}&\|J_z e\|+|s-1|\,\|b_z\|\\
 &+\frac\beta2\sum_i|g_i|\,\|h_i-z\|^2.
\end{aligned}
\label{eq:bpc-local-error-bound}
\end{equation}
\end{proposition}
\begin{proof}
Taylor expansion gives $\pi(h_i)=\pi(z)+J_z(h_i-z)+r_i$,
with $\|r_i\|\leq\beta\|h_i-z\|^2/2$. Consequently,
\begin{equation}
 Ag-\pi(z)=J_z e+(s-1)b_z+\sum_i g_i r_i.
\label{eq:bpc-taylor-identity}
\end{equation}
Taking norms proves the claim.
\end{proof}

The terms in (\ref{eq:bpc-local-error-bound}) identify three sources of error: action-relevant history mismatch, affine mass inconsistency, and curvature. Under specific constraints on $g$, some of these terms may collapse (Corollary 1).

\begin{corollary}[Affine continuation]
If $\mathbf1^\top g(z)=1$, then
\begin{equation}
\begin{aligned}
 \|Ag(z)-\pi(z)\|
 & \leq \|J_z(Hg(z)-z)\|
       +\frac{\beta\rho^2}{2}\|g(z)\|_1.
\label{eq:bpc-affine-bound}
\end{aligned}
\end{equation}
\end{corollary}

\Cref{eq:bpc-affine-bound} suggests that the demonstrations we use to predict expert actions should supply a local support for correcting the live history, rather than merely a collection of similar demonstrations. The Jacobian \(J_z\) determines which history differences translate into action error, motivating a retrieval metric (i.e. how to choose the best demonstrations from $\mathcal{D}$) that emphasizes differences predictive of the expert's continuation rather than observational similarity. The residual \(Hg-z\), meanwhile, motivates selecting histories whose complementary variations can explain the query, particularly along these action-relevant directions. This reconstruction must be balanced against the curvature cost \(\beta\rho^2\|g\|_1\): \(\beta\) bounds how rapidly the history-to-action Jacobian varies, while \(\rho=\max_i\|h_i-z\|\) measures how far the retrieved histories extend from the query. Nearby histories reduce this radius and may admit a tighter local curvature bound. However, if the histories are nearly redundant, or the query lies far beyond their available spread, reconstruction may require large positive and negative coefficients. The resulting increase in \(\|g\|_1\) amplifies the Taylor remainders; this term measures signed coefficient amplification. A favorable support should therefore be local and action-aware, yet retain enough variation in the directions needed to reconstruct the query without excessive cancellation.

This balance explains why the smallest neighborhood is not necessarily the most useful. A single neighbor has \(\|g\|_1=1\), but cannot correct any history mismatch remaining. Uniform or distance-weighted \(k\)NN averaging also keeps \(\|g\|_1=1\), yet its weights do not explicitly fit the live history: adding neighbors can enlarge the radius over which curvature contributes without eliminating the first-order mismatch. Moreover, nonnegative averaging cannot reconstruct a query outside the retrieved convex hull. Affine continuation instead uses differences among demonstrated histories to reduce that mismatch, including through modest extrapolation, at the cost of potentially greater coefficient amplification. Retrieval and affine continuation are thus complementary: a selective but sufficiently rich neighborhood can make first-order correction possible while keeping its nonlinear cost small. The objective is not maximal sparsity or proximity alone, but a support whose histories are jointly useful for regularized reconstruction. Since \(J_z\) is unknown, demonstrated continuations provide a natural supervision signal for learning the action-aware component of this retrieval metric, motivating the future-informed retrieval developed next.

\begin{figure}[]
    \centering
    \captionsetup{font=footnotesize}
\includegraphics[width=0.5\textwidth]{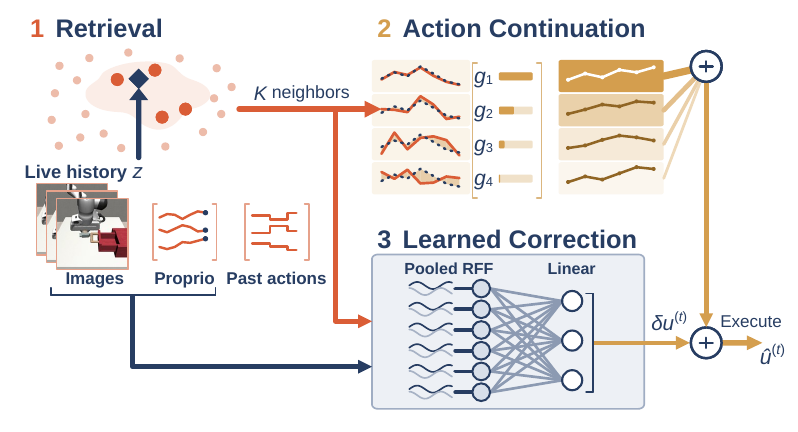}%
    \caption{\textbf{Online BPC pipeline.}
    (1) A learned metric retrieves $K$ demonstration windows from the live
    observation--action history $z$.
    (2) Regularized reconstruction fits coefficients $g_i$ to their past
    histories and applies them to future action blocks to form the action
    prior. (3) Random Fourier features of retrieved observations, actions,
    and observation differences from the query are uniformly pooled.
    A linear head predicts the correction $\delta u^{(t)}$. The action prior is added to the correction to form the outgoing control action.}
    \label{fig:bpc_pipeline}
    \vspace{-15pt}
\end{figure}

\makeatletter
\@ifundefined{ifbpcincludeactioncalibration}{%
  \newif\ifbpcincludeactioncalibration
  \bpcincludeactioncalibrationfalse
}{}
\makeatother

\section{Extension to Nonlinear Behaviors}
\label{sec:bpc-method}

This perspective motivates two complementary components: future-informed retrieval to select the support for local continuation, and a shared-support nonlinear residual to correct errors that remain. Because the action sensitivity \(J_z\) is unknown, we use demonstrated continuations to fit a retrieval metric that favors histories predictive of valid future actions, while requiring only the observed history at deployment. Such retrieval seeks a suitable neighborhood for the affine prior, but does not eliminate errors from curvature or imperfect reconstruction. We therefore augment the prior with a nonlinear correction conditioned on the same retrieved demonstrations and their differences from the current observation. Fixed random Fourier features supply the nonlinear representation, while ridge regression fits the correction weights without training a nonlinear network.

\begin{algorithm}[t]
\caption{Behavior Predictive Control}
\label{alg:bpc-online}
\footnotesize
\begin{algorithmic}[1]
\Require Demonstration bank $\mathcal{D}$ and current history $z$.
\State \textbf{Fit (offline):} fit retriever (\ref{eq:bpc-retrieval-geometry}), then fit $W$ to residual targets (\ref{eq:bpc-ridge-objective}).
\Statex \textit{Online, at each time step:}
\State \textbf{Retrieval:} select candidates $\mathcal N$ using action ridge (\ref{eq:bpc-action-ridge-fit})  or sparsemax-LDA (\ref{eq:bpc-lda-statistics}).
\State Assemble local candidate banks $H,U$; set $A\gets U^{(1)}$.
\State \textbf{Continuation:} fit $g$ to $Hg\approx z$ (\ref{eq:base-policy}).
\State \textbf{Correction:} pool $\bar\varphi$ over candidate bank~\eqref{eq:bpc-pooled-feature}.
\State $\hat{u}\gets Ag+W^\top\bar\varphi$.
\end{algorithmic}
\end{algorithm}
\subsection{Future-informed retrieval}
\label{subsec:bpc-retrieval}

There exist many ways to perform retrieval, so we ground the discussion by providing a unifying framework. Given the current ongoing history $z$ and the data bank of histories $H$, we would like to find a
transformation $L$ and a squared dissimilarity $d(\cdot)$ that appropriately ranks columns in $\mathcal{H}_\mathcal{D}$. In particular, we use a squared Mahalanobis dissimilarity~\cite{weinberger2009distance} to compare the observed history to the training dataset in some feature space
\begin{equation}
\begin{aligned}
 \Psi(z)&=L^\top\psi(z),\qquad \Phi(h)=L^\top\phi(h),\\
 d_i(z)&=\|\Psi(z)-\Phi(h_i)\|^2,
\end{aligned}
\label{eq:bpc-retrieval-geometry}
\end{equation}
where $\psi$ is a potentially nonlinear history feature map and $\phi_i$
is the feature representation of candidate $i$.
Retrieval favors candidates with smaller distances in this transformed space~\cite{snell2017prototypical}. 

We propose two different retrieval pipelines based on \Cref{eq:bpc-retrieval-geometry}: action ridge regression and linear discriminant analysis (LDA). Either one can be used depending on the nature of the data.

\paragraph*{Action ridge}
The action-aware ridge regression fits the full  history $H$ and future-action
banks $U$:
\begin{equation}
\begin{aligned}
 L&=\underset{\hat{L}}{\operatorname{argmin}}\;
 \|\hat{L}^\top H-U\|_F^2+\lambda\|\hat{L}\|_F^2\\
  &=(HH^\top+\lambda I)^{-1}HU^\top,\\
  \psi(z) &= z, \qquad \phi(h_i) = h_i.
\end{aligned}
\label{eq:bpc-action-ridge-fit}
\end{equation}
Here $\lambda>0$ regularizes the predictor. Equation~\eqref{eq:bpc-retrieval-geometry}
then compares the \emph{predicted future actions} of the query and bank
history. Histories are close when the predictor expects similar
continuations. Note that this linear prediction is only used for retrieval (i.e. ranking) and not for action prediction.

\paragraph*{Linear discriminant analysis}
LDA compares demonstrated futures with one another.
Let $d_{ij}$ be the squared distance between the future following training
history $h_i$ and the future following $h_j$. Collect these distances
in $d_i$, and let $b_i>0$ be the squared distance to the eligible
future neighbor of rank $k_{\mathrm{loc}}$, floored at
$\varepsilon>0$. Here $k_{\mathrm{loc}}$ controls the neighborhood
scale. We apply sparsemax~\cite{martins2016sparsemax} to turn future action similarity into sparse supervision:
\begin{equation}
 T_{ij}=\bigl[\operatorname{sparsemax}(-d_i/b_i)\bigr]_j.
\label{eq:bpc-action-teacher}
\end{equation}
These nonnegative weights sum to one over $j$ and are high when futures
are similar. For each candidate continuation $j$, we use the weights
$T_{ij}$ to average the fitting-history features $\varphi(h_i)$.
We use random Fourier features~\cite{rahimi2007random} for
$\varphi(\cdot)$.

Let $x_i=\varphi(h_i)$, $\bar{x}=n^{-1}\sum_i x_i$, and
$m_j=\sum_i T_{ij}$, where $n$ is the number of fitting histories.
The class means and within- and between-class covariances are
\begin{equation}
\begin{aligned}
 \mu_j&=\frac{\sum_i T_{ij}x_i}{m_j},\\
 S_{\mathrm w}&=\frac1n\sum_{i,j}T_{ij}
 (x_i-\mu_j)(x_i-\mu_j)^\top,\\
 S_{\mathrm b}&=\frac1n\sum_j m_j
 (\mu_j-\bar{x})(\mu_j-\bar{x})^\top.
\end{aligned}
\label{eq:bpc-lda-statistics}
\end{equation}
Zero-mass classes are omitted from these covariance sums.
Let $\ell_1,\ldots,\ell_r$ be the leading generalized Fisher directions:
\begin{equation}
\begin{aligned}
 S_{\mathrm b}\ell_k&=\lambda_k(S_{\mathrm w}+\eta I)\ell_k,\\
 L&=[\,\ell_1\ \cdots\ \ell_r\,],
 \qquad L^\top(S_{\mathrm w}+\eta I)L=I,\\
 \psi(z)&=\varphi(z),\qquad \phi(h_j)=\varphi(h_j),
\end{aligned}
\label{eq:bpc-fisher-projection}
\end{equation}
where $\eta>0$ regularizes the within-class covariance and $r$ is the
chosen projection dimension; the $r$ largest eigenvalues are retained.
The means $\mu_j$ summarize pasts associated with continuation $j$
and determine the fitted metric; retrieval compares the query with
each candidate's own history features.
The projection favors directions with large between-class variation
relative to within-class variation~\cite[Sec.~4.3.3]{hastie2009elements}.


Intuitively, action ridge is a natural choice when differences in future actions are well captured by a linear map of the history. LDA is a natural choice when histories with similar futures exhibit substantial variation.

\paragraph*{Selecting the local dictionary}
Both fits assign a distance to each candidate continuation. Action ridge
retains the $K$ nearest candidates. LDA converts negative squared retrieval distances to selection weights
using sparsemax:
\begin{equation}
 q(z)=\operatorname{sparsemax}\!\left(
 -\alpha\,\frac{d(z)}
 {\max\{d_{(k_{\mathrm{loc}})}(z)-d_{(1)}(z),\varepsilon\}}
 \right),
\label{eq:bpc-adaptive-selection}
\end{equation}
where $\alpha\geq0$, $\varepsilon>0$, and $d_{(k)}(z)$ is the
$k$th smallest eligible squared retrieval distance.
We retain every candidate with $q_i>0$, so the number selected can vary
with the query.

\subsection{Composing Retrieval Distributions}
Both retrieval methods above score candidates through a distance metric $d_i(z)$. Re-interpreting that score as a marginal distribution $p(i\mid z)$, taking any decreasing function of $d_i(z)$, lets additional modes of the retrieved data enter as further marginals and combine by multiplication. Any modality whose structure is known can then be catered to the nature of the task -- a flexibility unavailable to a parametric policy whose inductive biases are fixed at architecture design time. For example, in non-Markovian tasks (See Section~\ref{subsec:manipulation-hardware}), a second factor over task progress can be introduced. We define $\tilde t_i\in[0,1]$ as candidate $i$'s normalized progress within its own demonstration and remain stateful by bootstrapping an estimate $\hat t$ to the $g$-weighted mean progress of the previous step's supports, information that is discarded by a parametric policy. Formally, the task-progress prior is a Laplace distribution and the posterior then becomes:
\begin{equation}
p(i\mid z,\hat t)\;\propto\;
p(i\mid z)\,
\exp\!\left(-\frac{\lvert\tilde t_i-\hat t\rvert}{\tau}\right).
\label{eq:bpc-phase-posterior}
\end{equation}
Selection then proceeds as before on this history-aware re-ranked posterior.

\subsection{Shared-support nonlinear residual}
\label{subsec:bpc-residual}

The continuation prior can leave action errors even when its history fit
is accurate (\ref{eq:bpc-affine-bound}). Inspired by difference-aware retrieval~\cite{pfeifer2026darp},
we fit a correction using the same retrieved candidates.

Let $y$ be the latest observation in query $z$
and $\mathcal N$ the nonempty set of retrieved candidates. For each candidate $i$, let $y_i$ be
its most recent observation and $a_i$ its first action thereafter.
Define
\begin{equation}
 x_i=\operatorname{col}\left(y_i,\ a_i,\ \frac{y_i-y}{\sqrt2}\right).
\label{eq:bpc-residual-evidence}
\end{equation}
We employ random Fourier features as the nonlinear representation~\cite{rahimi2007random},
\begin{equation}
 \varphi(x)=\sqrt{\frac2D}\cos(\Omega^\top x+\vartheta),
\label{eq:bpc-rff}
\end{equation}
with independent frequency columns drawn from
$\mathcal N(0,\sigma_\varphi^{-2}I)$ and phases from
$\operatorname{Unif}[0,2\pi]$. Pool uniformly over the retrieved candidates:
\begin{equation}
 \bar\varphi=\frac1{|\mathcal N|}\sum_{i\in\mathcal N}\varphi(x_i).
\label{eq:bpc-pooled-feature}
\end{equation}
The correction model is simply $W^\top\bar\varphi$, where
$W\in\mathbb R^{D\times n_u}$.

\paragraph*{Fitting the correction}
For each eligible training pair $(h_i,a_i)$, exclude candidates within the same demonstration that have temporal overlap with the training pair. Let $G^{-i}$ denote the  continuation prior on
this filtered support, and define
\begin{equation}
 r_i=a_i-G^{-i}(h_i),
\label{eq:bpc-residual-target}
\end{equation}
an in-dataset measure of how well the action-continuation model can reconstruct expert actions without having access to highly correlated candidates to the query. The pooled feature on this filtered support is computed the same way as (\ref{eq:bpc-pooled-feature}). 

For all $i$ in $\mathcal{D}$, stack $\bar\varphi_i^\top$ and $r_i^\top$ as the rows of $X$ and $R$.
For $\lambda>0$, fit
\begin{equation}
\begin{aligned}
 W&=\underset{\widehat W}{\operatorname{argmin}}\;
 \|X\widehat W-R\|_F^2+\lambda\|\widehat W\|_F^2,\\
 &= (X^\top X+\lambda I)^{-1} X^\top R.
\label{eq:bpc-ridge-objective}
\end{aligned}
\end{equation}

The BPC model is then
\begin{equation}
 \hat{u}=Ag+W^\top\bar\varphi.
\end{equation}

In summary, Behavior Predictive Control (Fig. \ref{fig:bpc_pipeline}, Alg. \ref{alg:bpc-online}) consists of three compositional components: (1) Action Aware Retrieval, (2) Local Action Continuation, and (3) Learned Residual Correction. Retrieval supplies support for local action continuation, while learned residual correction adjusts the predicted action to account for errors that remain after history fitting.

\section{Results}

\subsection{Tabletop Manipulation}
\label{subsec:manipulation-hardware}
We test BPC on several hardware manipulation tasks using the UFactory xArm6. Full task descriptions are as follows:
\begin{itemize}
    \item \textbf{Coffee}: Pick a bag of beans from a 30 cm by 30 cm area and place it on the drawer.
    \item \textbf{Drawer}: Open the drawer using the small handle, pick the sponge, place the sponge in the drawer, and then close the drawer.
    \item \textbf{Switch}: Pick the bag of beans and place it on the opposite plate (either the left or right side).
\end{itemize}
We compare our policy to $\pi_{0.5}$ \cite{pi05_2025} fine-tuned on the same demonstrations per task, with overall success rates denoted in Table~\ref{tab:tabletop-hardware}. All methods only use visuomotor information (wrist and external cameras and proprioceptive). While $\pi_{0.5}$ took $20$ hours to train, BPC takes only 30 to 120 seconds (depending on demonstration count) to fit on a $24$ GB RTX 4090. Crucially, we found using $16 \times 16$ \emph{raw pixels} (i.e. no image encoding) was sufficient for BPC to perform well. For the hardware manipulation experiments, we elected to use \emph{action ridge} retrieval, but all subsequent results utilize \emph{LDA}. 

In \textbf{Coffee}, both policies ground their actions in the varying bag position, achieving 36 and 39 successes out of 40 trials, respectively. In \textbf{Drawer}, both policies stay on-manifold for the minute-long task and achieve 38 and 40 successes out of 40 trials, respectively. In \textbf{Switch}, $\pi_{0.5}$'s performance degrades due to an architectural lack of memory, leading to a known phenomenon in single-observation behavior-cloning called \textit{perceptual aliasing} \cite{osterberg2026unimemunifyingmultimodalmemory}. In this case, once in the middle of transporting the bag, $\pi_{0.5}$ cannot remember which plate it originated from. Our method offers solutions to non-Markovianity in two ways: (1) retrieving over a history of states to maintain short-term orientation and (2) temporal filtering to disambiguate long-term progress. The first solution largely accounts for BPC's superior performance (34 versus 8 successes out of 40 trials) on the \textbf{Switch} task.

\begin{table}[t]
  \centering
  \captionsetup{font=footnotesize}
  \caption{Tabletop manipulation hardware: qualitative task sequences and
  successes over number of trials.}
  \vspace{-3pt}
  \label{tab:tabletop-hardware}
  \footnotesize
  \setlength{\tabcolsep}{2pt}
  \renewcommand{\arraystretch}{1.1}
  \definecolor{hardwarestagefill}{HTML}{DFE2E7}
  \definecolor{hardwarestageink}{HTML}{273D62}
  \newcommand{\hardwareframe}[2]{%
    \begingroup
    \setlength{\unitlength}{\linewidth}%
    \begin{picture}(1,0.5625)
      \put(0,0){\includegraphics[width=\unitlength]{#2}}
      \put(0.10,0.4625){\color{hardwarestagefill}\circle*{0.15}}
      \put(0.10,0.4625){\makebox(0,0){%
        \color{hardwarestageink}\scriptsize\sffamily\bfseries #1}}
    \end{picture}%
    \endgroup}
  \begin{tabular}{@{}*{3}{>{\centering\arraybackslash}m{\dimexpr(\columnwidth-4\tabcolsep)/3\relax}}@{}}
    \toprule
    \textbf{Coffee} & \textbf{Drawer} & \textbf{Switch} \\
    \midrule
    \hardwareframe{1}{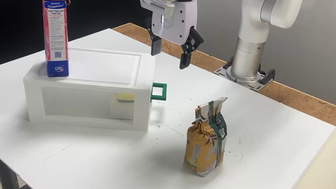}
    & \hardwareframe{1}{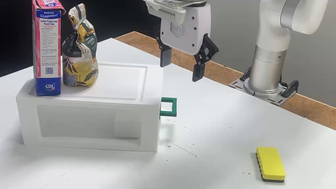}
    & \hardwareframe{1}{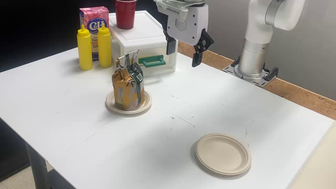} \\[2pt]
    \hardwareframe{2}{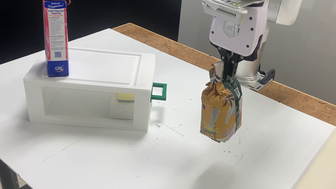}
    & \hardwareframe{2}{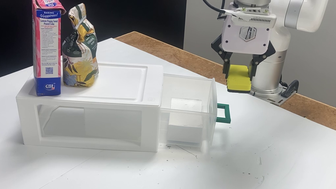}
    & \hardwareframe{2}{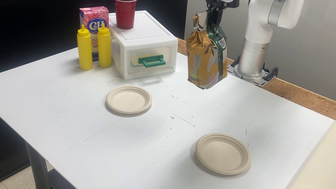} \\[2pt]
    \hardwareframe{3}{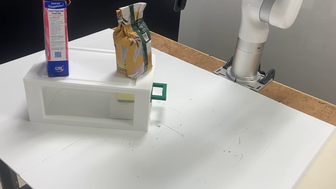}
    & \hardwareframe{3}{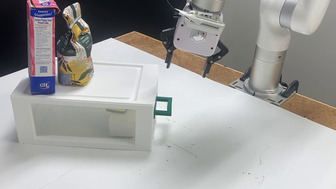}
    & \hardwareframe{3}{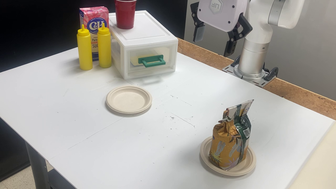} \\
    \bottomrule
  \end{tabular}
  \par\vspace{3pt}
  \begin{tabular}{@{}m{\dimexpr0.4\columnwidth-3\tabcolsep\relax}*{3}{>{\centering\arraybackslash}m{\dimexpr0.2\columnwidth-\tabcolsep\relax}}@{}}
    \textbf{Method} & \textbf{Coffee} & \textbf{Drawer} & \textbf{Switch} \\
    \midrule
    $\pi_{0.5}$ & \textbf{39/40} & \textbf{40/40} & 8/40 \\
    \rowcolor[HTML]{F4F5F7} BPC (Ours) & 36/40 & 38/40 & \textbf{34/40} \\
    \bottomrule
  \end{tabular}
  \vspace{-5pt}
\end{table}
\begin{table}[t]
  \centering
  \captionsetup{font=footnotesize}
  \caption{Success rates (\%) on simulation tabletop manipulation tasks over 300 rollouts per task.}
  \label{tab:tabletop-tasks}
  \fontsize{6}{7}\selectfont
  \setlength{\tabcolsep}{0.7pt}
  \renewcommand{\arraystretch}{1.05}
  \begin{tabular}{@{}>{\raggedright\arraybackslash}m{0.15\columnwidth}*{8}{>{\centering\arraybackslash}m{\dimexpr(0.79\columnwidth-18\tabcolsep)/8\relax}}>{\centering\arraybackslash}m{0.06\columnwidth}@{}}
      \toprule
      \textbf{Method} & \textbf{Square} & \textbf{Stack} & \textbf{Coffee} & \textbf{Hammer} & \textbf{Mug} & \textbf{Nut} & \textbf{Stack3} & \textbf{Thread} & \textbf{Mean} \\
      \midrule
       & \includegraphics[width=\linewidth]{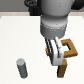}
      & \includegraphics[width=\linewidth]{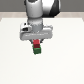}
      & \includegraphics[width=\linewidth]{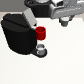}
      & \includegraphics[width=\linewidth]{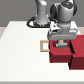}
      & \includegraphics[width=\linewidth]{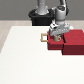}
      & \includegraphics[width=\linewidth]{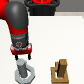}
      & \includegraphics[width=\linewidth]{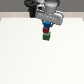}
      & \includegraphics[width=\linewidth]{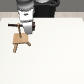} & \\
      \midrule
      SDP~\cite{wang2024sdp} & 74 & 99 & 83 & 98 & 70 & \textbf{42} & 76 & 65 & 76 \\
      $\pi0.5$ & \textbf{86} & \textbf{100} & 74 & \textbf{99} & 68 & 31 & \textbf{85} & 72 & \textbf{77} \\
      DARP~\cite{pfeifer2026darp} & 73 & 80 & 48 & \textbf{99} & 40 & 2 & 4 & 67 & 52 \\
      \rowcolor[HTML]{F4F5F7} BPC (Ours) & 71 & 95 & \textbf{97} & \textbf{99} & \textbf{78} & 12 & 62 & \textbf{73} & 73 \\
      \bottomrule
  \end{tabular}
  \vspace{-5pt}
\end{table}
We also demonstrate the effectiveness of BPC on simulated visuomotor MimicGen~\cite{mandlekar2023mimicgen} tasks, shown in \Cref{tab:tabletop-tasks}. We compare performance against DARP \cite{pfeifer2026darp}, Sparse Diffusion Policy (SDP) \cite{wang2024sdp}, and $\pi_{0.5}$. We re-use SDP's reported success rates in the table. We find that BPC with the R$3$M \cite{nair2023r3m} image encoding is competitive with the state-of-the-art \emph{learned} policies, and superior to other retrieval policies, despite BPC requiring $1 - 5$ minutes to fit 1000 training demonstrations compared to hours (SDP and $\pi_{0.5}$). It is also quicker than DARP with R$3$M (dozens of minutes), as it parametrizes the residual correction using an MLP. We find that the state-of-the-art policies are not monotonically better than BPC (e.g. \textbf{Coffee}), which motivates future work investigating missing mechanisms within the BPC framework. 

\subsection{Drone Navigation}
We evaluate BPC in a visuomotor quadrotor gate-navigation task (Fig. \ref{fig:drone-trajectory}) to demonstrate the policy's efficacy in the data-sparse regime and moderate robustness to perceptual aliasing. Like in the hardware tabletop manipulation tasks, BPC uses \emph{raw pixels} (no image encoding) as visual inputs to the policy. The first task is a simple \textbf{Loop} through the two gates. The second is a \textbf{Figure-eight} with greater variation in altitude and significant visual overlap at varying points in the task. We evaluate the performance of BPC in terms of task success rate and mean distance to the demonstration set during policy  (\Cref{tab:drone-hardware}). With only 3-4 demos each, BPC is able to replicate policy behavior with approximately 11cm of average spatial distance from the training demonstrations during test time. An example timestep and retrieved candidates from of each of the tasks is illustrated in Figure \ref{fig:drone-trajectories}.

\begin{table}[t]
\centering
\captionsetup{font=footnotesize}
\caption{Drone hardware results for \textbf{Loop} and \textbf{Figure-eight}.
Demonstration-set distance is the mean 3D distance to the nearest
demonstration trajectory. Demo duration is the mean duration per
demonstration.}
\vspace{-3pt}
\label{tab:drone-hardware}
\footnotesize
\setlength{\tabcolsep}{3pt}
\renewcommand{\arraystretch}{1.15}

\noindent\resizebox{\columnwidth}{!}{%
\begin{tabular}{@{}lcccc@{}}
\toprule
\textbf{Task}
& \shortstack{\textbf{Success}\\\textbf{rate}}
& \shortstack{\textbf{Collision}\\\textbf{rate (c/m)}}
& \shortstack{\textbf{Demo-set}\\\textbf{distance (m)}}
& \shortstack{\textbf{Demos / demo}\\\textbf{duration (s)}} \\
\midrule
Loop  & 10/10 & 0.0000 & 0.095 & 3 / 25.2 \\
Figure-eight & 8/10  & 0.0108 & 0.109 & 4 / 33.5 \\
\bottomrule
\end{tabular}%
}
\vspace{-13pt}
\end{table}

\begin{figure}[t]
\centering
\includegraphics[width=1.0\columnwidth]{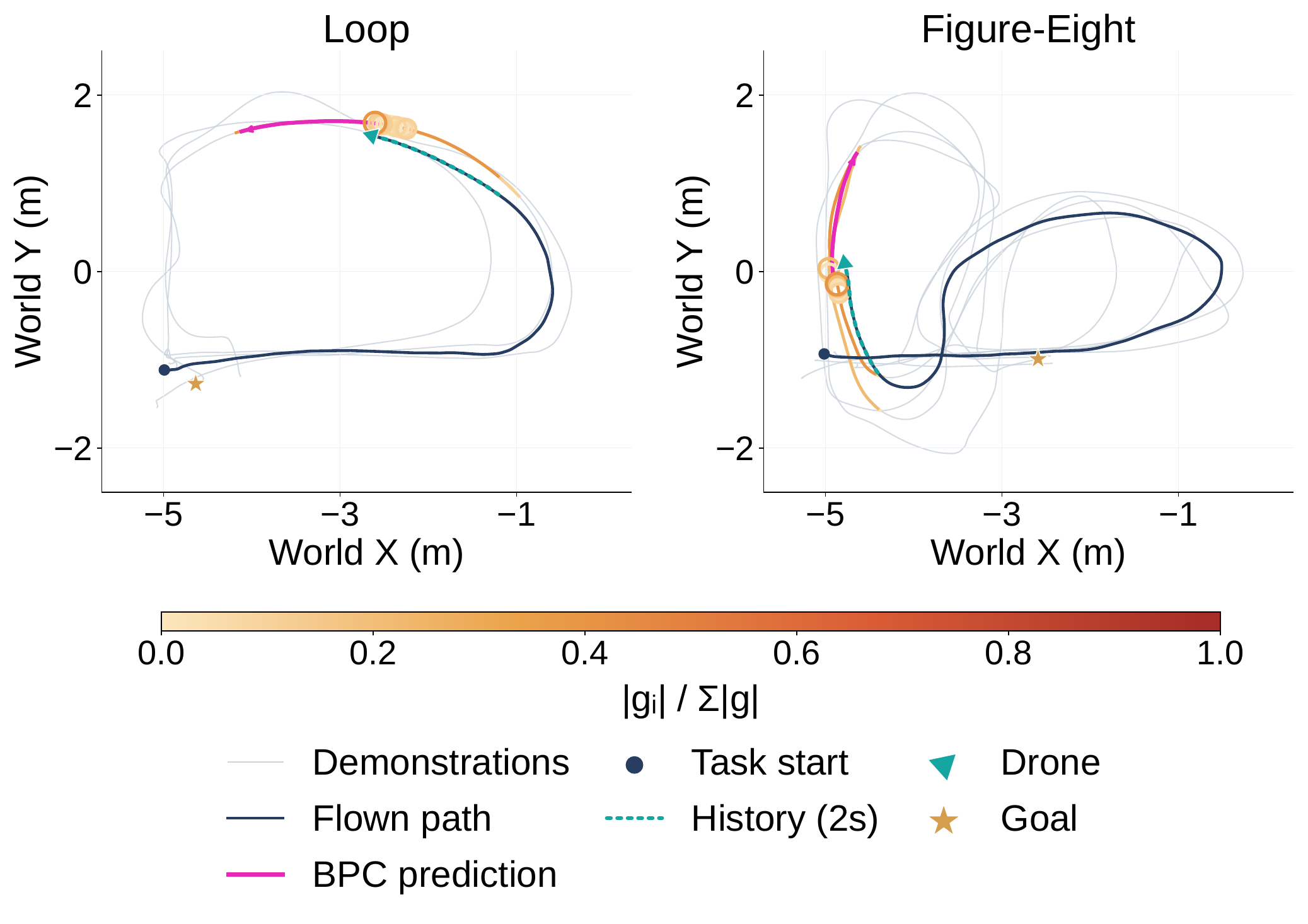} %
\captionsetup{font=footnotesize} 
\caption{Representative hardware trajectory snapshots for \textbf{Loop} (left) and \textbf{Figure-eight} (right). Demonstrations, flown paths, and predicted references are shown alongside the preceding two second history. Colored demonstration segments indicate normalized coefficient magnitudes, $|g_i|/\sum_j |g_j|$.}
\label{fig:drone-trajectories}
\vspace{-10pt}
\end{figure}

\textbf{Figure-eight} has perceptual aliasing at the crossover in the center of the trajectory, where the drone approaches the inside of the two gates, and at the beginning and end of the task, where the drone must navigate through the bottom and top of the left gate respectively. Out of ten trials, the only two failures in  \textbf{Figure-eight} are attributable to perceptual aliasing. The drone repeated a section of the trajectory in one failure, and entered the lower gate at the end instead of the upper gate in the other.

\begin{figure}[t]
  \centering
  \includegraphics[width=0.9\linewidth]{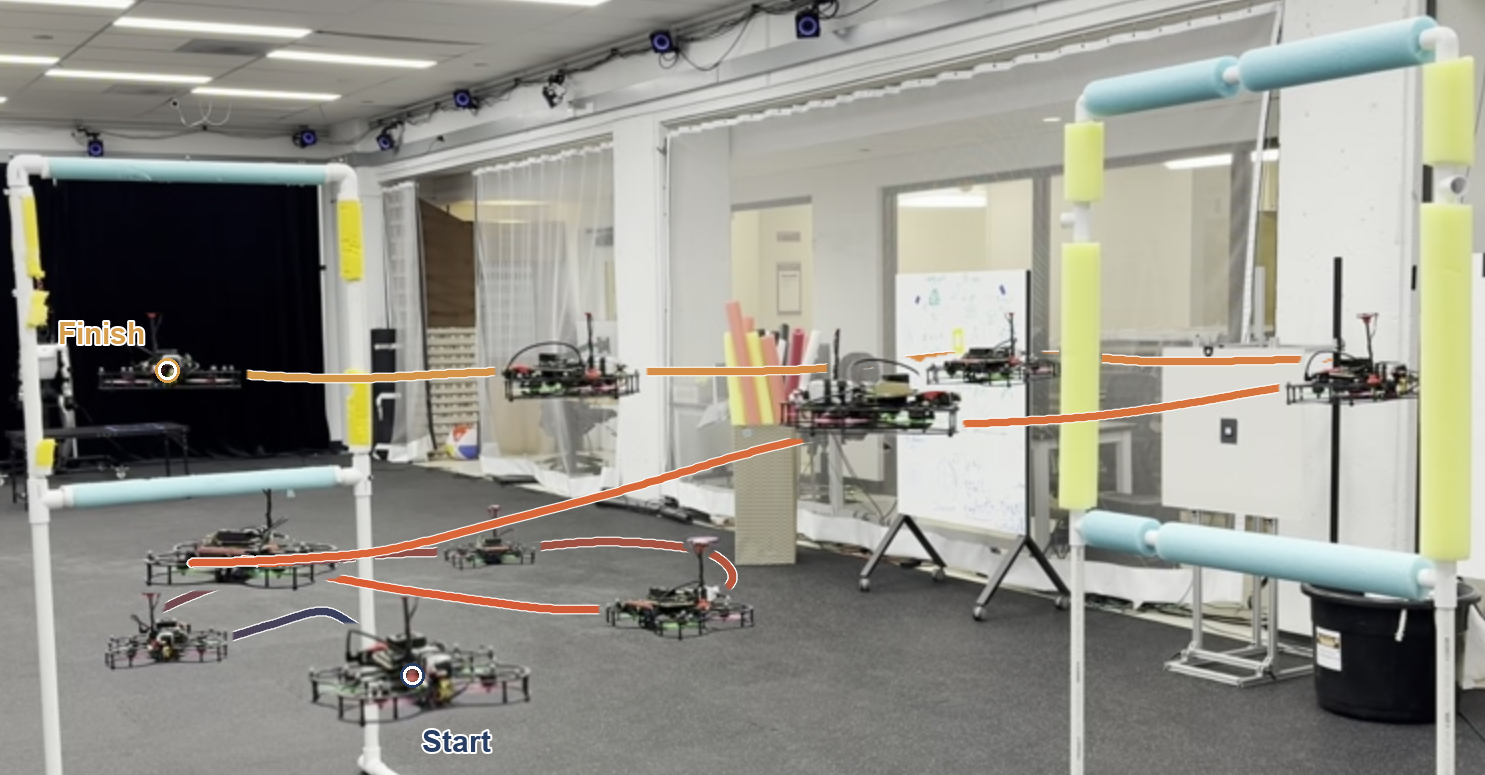} %
  \captionsetup{font=footnotesize}
  \caption{Drone navigation trajectory visualization.}
  \label{fig:drone-trajectory}
  \vspace{-5pt}
\end{figure}

\subsection{Dexterous Manipulation}
\begin{table}[t]
  \centering
  \captionsetup{font=footnotesize}
  \caption{Success rates (\%) on the dexterous manipulation suites with 120 rollouts per task: DexArt~\cite{bao2023dexart} and Adroit~\cite{rajeswaran2018dexterous}.}
  \label{tab:dexart-adroit}
  \fontsize{6}{7}\selectfont
  \setlength{\tabcolsep}{0.7pt}
  \renewcommand{\arraystretch}{1.05}
  \newcommand{\dextaskwidth}{\dimexpr(0.67\columnwidth-16\tabcolsep)/6\relax}
  \begin{tabular}{@{}>{\raggedright\arraybackslash}m{0.20\columnwidth}*{3}{>{\centering\arraybackslash}m{\dextaskwidth}}>{\centering\arraybackslash}m{0.065\columnwidth}*{3}{>{\centering\arraybackslash}m{\dextaskwidth}}>{\centering\arraybackslash}m{0.065\columnwidth}@{}}
      \toprule
      & \multicolumn{4}{c}{DexArt} & \multicolumn{4}{c}{Adroit} \\
      \cmidrule(lr){2-5} \cmidrule(lr){6-9}
      \textbf{Method} & \textbf{Toilet} & \textbf{Faucet} & \textbf{Laptop} & \textbf{Mean} & \textbf{Door} & \textbf{Hammer} & \textbf{Pen} & \textbf{Mean} \\
      \midrule
        & \includegraphics[width=\linewidth,height=\dextaskwidth,keepaspectratio]{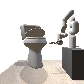}
       & \includegraphics[width=\linewidth,height=\dextaskwidth,keepaspectratio]{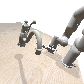}
       & \includegraphics[width=\linewidth,height=\dextaskwidth,keepaspectratio]{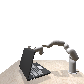}
       &
       & \includegraphics[width=\linewidth,height=\dextaskwidth,keepaspectratio]{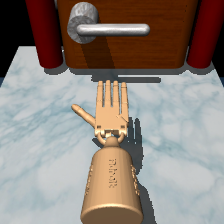}
       & \includegraphics[width=\linewidth,height=\dextaskwidth,keepaspectratio]{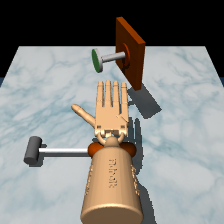}
       & \includegraphics[width=\linewidth,height=\dextaskwidth,keepaspectratio]{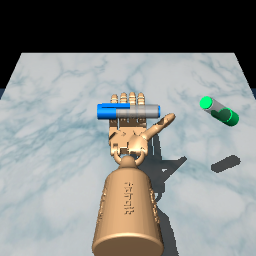}
       &
      \\
      \midrule
      SDP-3D~\cite{wang2024sdp} & 75 & 43 & \textbf{82} & 67 & 70 & \textbf{97} & \textbf{58} & \textbf{75} \\
      DARP-3D & \textbf{81} & 43 & 69 & 64 & \textbf{97} & 59 & 29 & 62 \\
      \rowcolor[HTML]{F4F5F7} BPC-3D (Ours) & \textbf{81} & \textbf{59} & 67 & \textbf{69} & 93 & 73 & 24 & 63 \\
      \bottomrule
  \end{tabular}
  \vspace{-5pt}
\end{table}
We demonstrate the flexibility of BPC on more complex dexterous manipulation tasks with a 3D encoding using PointNet~\cite{qi2017pointnet} instead of RGB images. Using 10 Adroit~\cite{rajeswaran2018dexterous} and 100 DexArt~\cite{bao2023dexart} demonstrations per task, we show competitive results with SDP-3D (success rates taken from \cite{wang2024sdp}) and better success rates than DARP with the same inputs (PointNet latents and proprioceptive). Success rates are averaged over 120 rollouts in \Cref{tab:dexart-adroit}.

\subsection{Ablations}
To test our assumptions about the contributions to the BPC pipeline, we remove parts of the pipeline and measure the success rates of the simpler visuomotor policy in MimicGen (\Cref{tab:ablation-tabletop}) using the same configuration as \Cref{tab:tabletop-tasks}. 
\textbf{NHC} stands for No History-conditioned Continuation, which replaces the action-continuation with a uniform-averaging kNN base policy. \textbf{NNC} denotes No Nonlinear Correction, keeping only the base policy. \textbf{NLR} signifies No Learned Retrieval, which chooses only the top-K neighbors using the raw L2 distance in the input space. We find that all pieces of the pipeline are synergistic and important, with the correction having the most impact, followed by the action continuation, and then the retrieval mechanism.

\begin{table}[t]
  \centering
  \captionsetup{font=footnotesize}
  \caption{Ablations on tabletop manipulation tasks: success rate (\%) over 300 trials per task.} 
  \vspace{-3pt}
  \label{tab:ablation-tabletop}
  \resizebox{\linewidth}{!}{%
    \begin{tabular}{lccccccccc}
      \toprule
      Method & Square & Stack & Coffee & Hammer & Mug & Nut & Stack3 & Thread & Mean \\
      \midrule
      NHC & 67 & 89 & 82 & 96 & 67 & 4 & 30 & \textbf{74} & 63 \\
      NNC & 65 & 6 & 90 & 94 & 24 & 6 & 1 & 52 & 42 \\
      NLR & \textbf{72} & 89 & 96 & \textbf{100} & 64 & 9 & 33 & 61 & 65 \\
      \rowcolor[HTML]{F4F5F7} BPC (Ours) & 71 & \textbf{95} & \textbf{97} & 99 & \textbf{78} & \textbf{12} & \textbf{62} & 73 & \textbf{73} \\
      \bottomrule
    \end{tabular}%
  }
\end{table}

\subsection{Computation Time}
Figure~\ref{fig:computation-scaling} reports fitting and policy-execution costs of BPC across different dataset scales.  Even at massive dataset scales, BPC only takes 3 minutes to fit while remaining well within the memory limits of consumer GPUs. It deploys at 100 Hz on an RTX 4090. We push BPC to the limits by fitting and deploying on an 8 GB Jetson Orin Nano, which takes less than 3 minutes to fit and was executed on-board the drone in \Cref{fig:drone-trajectory} at 75 Hz. 
\begin{figure*}[t]
  \centering
  \vspace{-10pt}
  \captionsetup{font=footnotesize}
  \includegraphics[width=\textwidth]{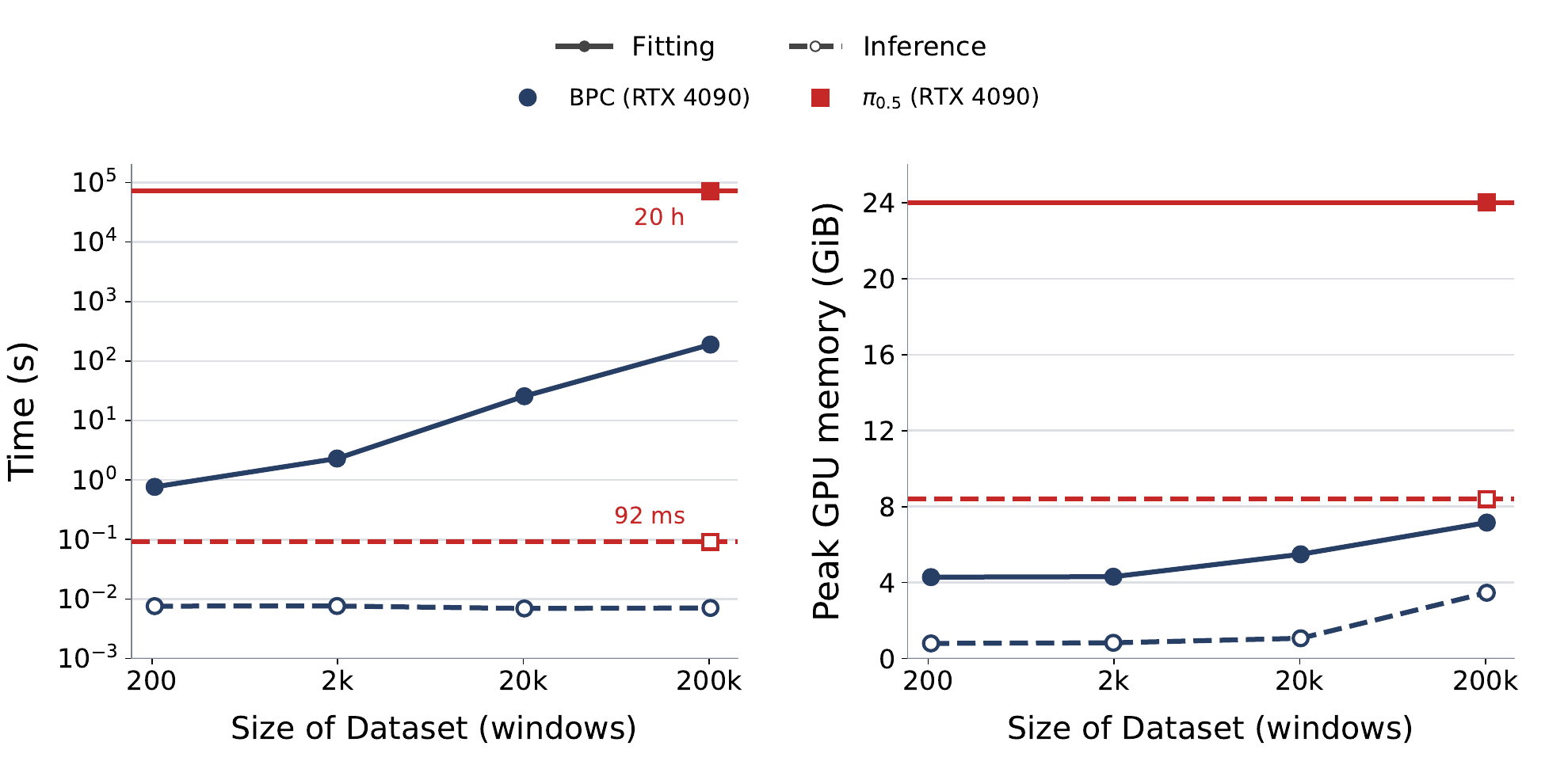}
  \caption{Computation scaling across different policies (BPC and $\pi_{0.5}$) across the number of windows in the bank.
  Left: fitting time and per-call execution latency (log scale); right: peak GPU memory.
  Solid lines and filled shapes show fitting, while dotted lines and open shapes show test-time values.
  Squares mark $\pi_{0.5}$ on an RTX 4090, and circles represent BPC values on an RTX 4090.
  }
  \label{fig:computation-scaling}
  \vspace{-17pt}
\end{figure*}

\section{Limitations}
BPC depends on demonstration coverage and on observation histories that
contain enough information to distinguish the required actions. Retrieval
cannot resolve ambiguities caused by missing task information, and behavior
far from the demonstrated states may require recovery actions absent from
the bank. Affine continuation permits extrapolation, but large signed
coefficients can amplify approximation error. Our analysis assumes a locally
smooth expert policy and bounds the uncorrected action prior; it does not
guarantee closed-loop stability, safety, or task completion after residual
correction. Specifically, we find that performance varies across tasks, with weak results on \textbf{Nut}, \textbf{Stack3}, and \textbf{Pen}. Further investigation is required to find mechanisms or dataset-specific quirks that BPC can exploit.  

Avoiding end-to-end policy training does not eliminate fitting or data
preparation. BPC still fits retrieval and residual components, and retaining the
demonstration bank also introduces storage and retrieval costs that grow
with the dataset. Although BPC is already compelling for large single-task datasets, future work will be dedicated to massively scaling this paradigm to multi-task and many more demonstrations. 

\section{Conclusion}
Behavior Predictive Control formulates behavior cloning as local action
continuation over explicit demonstration trajectories. It combines learned
retrieval, regularized history reconstruction, and a shared-support residual
to construct visuomotor policies without end-to-end policy training.
Simulation and hardware results show competitive performance compared to state-of-the-art end-to-end policies on several
manipulation and navigation tasks.
In fact, BPC can be fitted and deployed on consumer GPUs and even Jetson Orin Nanos. 
Retaining the demonstration bank
also makes action provenance accessible, enables direct behavior revision,
and provides an intrinsic estimate of task progress. We find that BPC's interpretability and modularity are highly synergistic with agentic loops, allowing agents to quickly deploy,  effectively retrieve useful feedback, surgically revise, and re-deploy. Finally, these properties motivate
further work on uncertainty under distribution shift, scalable bank updates and continual learning,
and generalization to multi-task deployments.

\section*{Acknowledgments}
Codex was used in drafting portions of the mathematics in Sections 3 and 4, and in preparing LaTeX formatting. Its image-generation system was used to generate the conceptual robot-arm and drone illustrations in Figure 1.

\bibliographystyle{IEEEtran}
\bibliography{main}

\clearpage
\appendix[Simulation Implementation and Hyperparameters]
\label{app:simulation}
This appendix specifies the BPC configurations used for the simulation results in Tables~\ref{tab:tabletop-tasks} and~\ref{tab:dexart-adroit}. The tabletop ablations in Table~\ref{tab:ablation-tabletop} use the same base configuration with the indicated component removed. These settings describe the reported benchmark runs, rather than later task-specific configurations. Hardware experiments are outside the scope of this appendix.


\subsection{Temporal Windows and Control Protocol}
The history length $T_{\mathrm{ini}}$ counts past observation--action steps, and the continuation horizon $T_f$ counts future action steps in each bank window. All simulation configurations use $T_f=10$. The history is 10 steps except for Square, which uses 5. Thus, a bank window spans 20 steps, or 15 for Square. The controller executes one action and updates its history before producing the next action; the 10-step continuation is not executed as an open-loop action chunk. The history contains the actions actually executed in the simulator.

\begin{table}[h]
\centering
\caption{Simulation data and temporal settings. Episode limits include demonstration-action warmup. Counts and horizons are in simulator control steps.}
\label{tab:appendix-protocol}
\footnotesize
\setlength{\tabcolsep}{3pt}
\renewcommand{\arraystretch}{1.12}
\begin{tabular}{@{}lccc@{}}
\toprule
Setting & MimicGen & Adroit & DexArt \\
\midrule
Demonstrations per task & 1,000 & 10 & 100 \\
History $T_{\mathrm{ini}}$ & 10 (Square: 5) & 10 & 10 \\
Continuation $T_f$ & 10 & 10 & 10 \\
Actions executed per update & 1 & 1 & 1 \\
Warmup & 10 (Square: 5) & 10 & 10 \\
Maximum episode length & 600 & 200 & 250 \\
Reported rollouts per task & 300 & 120 & 120 \\
\bottomrule
\end{tabular}
\end{table}

MimicGen uses independent reset seeds 100000--100299. After warmup, BPC controls at most 590 steps (595 for Square), with early stopping on task success. For Adroit and DexArt, four fitted metric seeds (0--3) are each evaluated on the same 30 reset seeds, 950000--950029. BPC controls up to 190 steps in Adroit and 240 in DexArt after the 10-step warmup. DexArt uses the seen-instance catalog. Dexterous success is measured by the native task predicate at the end of the rollout, not by the generic episode-termination flag. Actions are clipped to $[-1,1]$ in all three suites.

\subsection{Observation Encoders}
\textbf{MimicGen.} The policy receives an external agent-view image and a wrist-camera image. Each RGB image is resized to $224\times224$ and encoded by a frozen, pretrained R$3$M ResNet-18 into a $512$D feature. Concatenating the two image embeddings with the 9-dimensional robot proprioception gives a 1,033-dimensional policy observation. The robot proprioception consists of end-effector position (3), end-effector quaternion (4), and gripper joint positions (2). The action is operational-space end-effector deltas with 7 components. Unlike in related work \cite{pfeifer2026darp}, no privileged simulator object-state vector is supplied to the policy.

\textbf{Adroit and DexArt.} The reported BPC-3D configuration uses the DP3-style \cite{Ze2024DP3} PointNet encoder, with a 64-dimensional point-cloud embedding concatenated with robot proprioception. Adroit uses 24 proprioceptive components, yielding 88 observation components; Door, Hammer, and Pen have 28, 26, and 24 action components, respectively. DexArt uses 32 proprioceptive components, yielding 96 observation components, and 22 action components for Toilet, Faucet, and Laptop. Adroit uses 512 depth-derived XYZ points. DexArt uses 1,024 observed XYZ points augmented with 96 robot-model points (1,120 points in total). The PointNet applies pointwise layers of widths $3\rightarrow64\rightarrow128\rightarrow256$, max pooling, and a projection to 64 features, with layer normalization.

For each task, this encoder is trained from scratch on the selected demonstration subset using an auxiliary behavior-cloning objective for 2,000 updates, batch size 32, and learning rate $10^{-4}$. The auxiliary model uses two observation steps and predicts 16 action steps; these are encoder-training settings, distinct from BPC's 10-step history and continuation. Only the point-cloud encoder is retained for BPC, and it is frozen during retrieval/residual fitting and deployment. Thus, the reported 3D pipeline includes task-specific encoder training; it does not use released SDP weights or the public PointNet++ alternative. Frozen pre-trained PointNet latents achieve slightly lower success rates on these dexterous tasks. 

\subsection{Retrieval and Residual Hyperparameters}

Two distinct random-feature maps are used: the retrieval map in Linear Discriminant Analysis, and the residual-correction map. Both contain 16,384 features; the reduced retrieval representation retains the $r=70$ leading generalized Fisher directions. We use neighborhood rank $k_{loc} = 256$ for both future-action teacher normalization and query-distance normalization. These dimensions should not be confused with the image or point-cloud embedding dimensions.


The residual RFF bandwidth is half the median nonzero pairwise distance among sampled residual input vectors. The residual correction is fitted with ridge coefficient $\lambda=0.1$.

\end{document}